\documentclass[10pt,a4paper]{article}

\usepackage[T1]{fontenc}
\usepackage[utf8]{inputenc}
\usepackage[margin=1in]{geometry}
\usepackage{lmodern}
\usepackage{microtype}
\usepackage{amsmath,amssymb,amsthm}
\usepackage{graphicx}
\usepackage{booktabs}
\usepackage{array}
\usepackage{tabularx}
\usepackage{placeins}
\usepackage[numbers,sort&compress]{natbib}
\usepackage{xurl}
\usepackage{xcolor}
\usepackage{hyperref}
\usepackage[nameinlink,noabbrev]{cleveref}

\hypersetup{
    colorlinks=true,
    linkcolor=blue,
    citecolor=blue,
    urlcolor=blue,
    pdftitle={Language Models Without a Trainable Input Embedding Table:
    Learning from Fixed Minimal Binary Token Codes},
    pdfauthor={Andrey Bochkov}
}

\newtheorem{proposition}{Proposition}
\newtheorem{corollary}{Corollary}

\newcommand{\Ftwo}{\mathbb{F}_2}
\newcommand{\dmodel}{d_{\mathrm{model}}}
\newcommand{\bin}{\operatorname{bin}}
\newcommand{\GL}{\operatorname{GL}}
\newcommand{\rank}{\operatorname{rank}}
\newcommand{\spanof}{\operatorname{span}}
\newcommand{\col}{\operatorname{col}}
\newcommand{\ppl}{\operatorname{PPL}}
\newcommand{\score}[2]{$#1\,{\pm}\,#2$}

\title{Language Models Without a Trainable Input Embedding Table:
Learning from Fixed Minimal Binary Token Codes}
\author{Andrey Bochkov\thanks{
Correspondence: \href{mailto:andrey.bochkov@gmail.com}
{andrey.bochkov@gmail.com}.}}
\date{}

\begin{document}
\maketitle

\begin{abstract}
We study whether a decoder-only language model requires an independently
trainable input vector for every token. For a vocabulary of size $V$, an
injective fixed-length binary identifier requires
$K=\lceil\log_2 V\rceil$ bits. We replace the usual trainable
$V\times\dmodel$ input table with fixed minimal binary token codes and
a parameter-free tiled lift to model width. With $V=65{,}536$ and
$\dmodel=1024$, this supplies each token as a fixed 16-bit code and removes
67.1M trainable parameters, approximately 12.5\% of the untied learned-input
baseline. We also study a table-free implementation with one fixed
invertible affine recoding over $\Ftwo^{16}$. Across three training seeds,
32-layer models trained on approximately 16--17B tokens obtain mean
held-out perplexities of 2.44 for the learned-input baseline, 2.36 for
canonical binary codes, and 2.39 for affine-recoded codes. These descriptive
results do not establish statistical superiority or equivalence.
Standardized evaluation of released base checkpoints with the LM Evaluation
Harness adds commonsense, knowledge, and language-modeling benchmarks.
The three paper checkpoints show broadly similar, task-dependent
performance, while external SmolLM2 reference models are substantially
stronger on many tasks. Our conclusion is therefore limited to the
studied regime: a free trainable token-indexed input table is not required
to learn nontrivial language modeling. The Transformer still learns
continuous representations, and the output vocabulary projection remains
standard and trainable.
\end{abstract}

\section{Introduction}
\label{sec:introduction}

Neural language models commonly map each token ID to a learned continuous
vector before contextual processing
\citep{bengio2003neural,vaswani2017attention,radford2019language}.
For vocabulary size $V$ and model width $\dmodel$, this interface is
usually a trainable matrix
$$
E_{\mathrm{in}}\in\mathbb{R}^{V\times\dmodel}.
$$
Each token has its own independently adjustable row. This is a convenient
parameterization, but it is not the only way to communicate token identity
to a neural network.

We investigate a restricted alternative. Token ID
$t\in\{0,\ldots,V-1\}$ is represented by its fixed binary expansion
$$
c(t)=\bin_K(t)\in\{0,1\}^{K},
\qquad K=\lceil\log_2 V\rceil.
$$
A fixed tiled lift repeats this code to model width. The resulting
interface contains no trainable token vectors, component vectors, or
dedicated input projection. All learned processing occurs in the
otherwise unchanged Transformer and output head.

The term \emph{minimal} refers specifically to fixed-length injective
binary identifiers. It is a counting statement, not a lower bound on
the dimension of arbitrary real-valued representations, and it does
not claim that binary geometry is optimal for learning.

The experiment holds the tokenizer, backbone, output projection, and
training recipe fixed while removing the trainable input table.
The learned-input model has approximately 537.2M trainable parameters;
the fixed-code models have approximately 470.1M. The reduction is
67,108,864 trainable parameters. This is a component-removal comparison,
not a comparison at equal total parameter count.

We evaluate two code assignments. \emph{Canonical Binary-16} uses the
ordinary binary expansion of token IDs. \emph{Affine Binary-16} applies
one fixed invertible affine transformation over $\Ftwo^{16}$. The
canonical checkpoint stores a small frozen code table for convenience;
the affine checkpoint computes codes directly. Both assignments admit
table-free implementations. Their scientific difference is the code
assignment, not an inherent requirement to store a table.

The main validation experiment uses three training seeds per variant.
Mean validation perplexity is 2.44 for the learned-input baseline,
2.36 for canonical codes, and 2.39 for affine-recoded codes. The canonical
model and baseline see the same number of training tokens; the affine
model sees slightly fewer. We interpret these observations descriptively:
they demonstrate successful learning without a trainable input table,
but do not constitute a formal equivalence or superiority test.

We additionally evaluate released base checkpoints using the
EleutherAI LM Evaluation Harness \citep{evalharness}, without supervised
fine-tuning or chat prompting. These evaluations characterize both the
relative behavior of the input interfaces and the modest absolute
capabilities of the models. SmolLM2-135M and SmolLM2-360M
\citep{benallal2025smollm2} provide external context rather than matched
architectural controls.

\paragraph{Scope of the claim.}
Removing an explicit trainable lookup does not eliminate learned
representations. Shared Transformer weights still transform the fixed
codes into continuous features. Our claim concerns the absence of an
independently trainable token-indexed \emph{input table}, not the absence
of an embedding function in a broader mathematical sense. The output
projection remains vocabulary-sized, trainable, and untied.

\paragraph{Contributions.}
\begin{itemize}
    \item We study an exact-minimal binary token interface with a fixed
    lift and zero trainable input-interface parameters.
    \item We report a three-seed component-removal comparison in
    32-layer decoder-only models trained on approximately 16--17B tokens,
    with explicit parameter and storage accounting.
    \item We evaluate released base checkpoints on standardized
    language-modeling and multiple-choice benchmarks, separating matched
    architectural comparisons from external capability references.
    \item We distinguish code assignment, table-free implementation,
    and shared representation learning, and state the controls needed
    for stronger mechanistic conclusions.
\end{itemize}

\section{Related Work}
\label{sec:related_work}

\paragraph{Learned lexical representations.}
Trainable token representations are central to neural language modeling
\citep{bengio2003neural} and distributional representation learning
\citep{mikolov2013efficient,pennington2014glove}. Transformers retain
a learned lexical interface before contextual computation
\citep{vaswani2017attention,radford2019language}.
These developments establish the usefulness of learned embeddings;
they do not by themselves establish that independently trainable
input vectors are necessary in every model or training regime.

\paragraph{Factorization, sharing, and compositional codes.}
ALBERT factorizes the vocabulary-dependent embedding parameterization
\citep{lan2020albert}; adaptive inputs allocate representation capacity
across vocabulary frequency bands \citep{baevski2019adaptive};
DeFINE uses deep factorized input token embeddings
\citep{mehta2020define}. LightRNN shares learned components through a
compact allocation scheme \citep{li2016lightrnn}, and hash embeddings
combine shared learned vectors with compact token-specific
parameterizations \citep{svenstrup2017hash}.
Shu and Nakayama learn compositional codes and codebooks for reconstructing
word embeddings \citep{shu2018compositional}.

Our construction belongs to this broader family of constrained and
compositional input representations. It takes a particularly restrictive
form: code assignments are fixed, binary width is minimal for exact
identity, and the lift is fixed rather than learned. We do not claim that
binary codes, many-hot representations, or shared bases are new ideas.

\paragraph{ALONE and subword-based reconstruction.}
ALONE constructs word embeddings from a shared base embedding,
non-trainable word-specific filters, and a learned feed-forward network
\citep{takase2020alone}. Its binary-mask variant is especially close
to the idea of supplying fixed discrete token distinctions to shared
learned computation. Importantly, ALONE already avoids a conventional
collection of independently trainable word vectors.

Sasaki et al.\ reconstruct word embeddings using subword information,
memory-shared embeddings, and attention-based combination
\citep{sasaki2021subword}. This connects compact word representations to
the accumulation and weighting of shared components.

The distinction in our experiment is not the removal of representation
learning. Rather, we provide a fixed minimal identity code directly to
an otherwise unchanged decoder-only backbone, without a separate
learned embedding generator. Shared Transformer computation then
performs the learned transformation. This is an empirical study of a
restrictive input parameterization, not a claim to originate generated
or compositional embeddings.

\paragraph{Embedding-free bit tokens in vision.}
MaskBit directly uses bit-token representations for embedding-free image
generation \citep{weber2024maskbit}. It is a close precedent for replacing
a conventional embedding interface with binary token inputs.
Its tokens arise from learned image tokenization and binary quantization,
whereas our codes are deterministic identifiers of a fixed text
vocabulary. The present study concerns autoregressive language modeling,
an exact-minimal token-ID code, and a controlled comparison against a
learned-input text model. It does not claim to introduce embedding-free
bit-token generation in general.

\paragraph{Frozen visual token representations.}
Earlier work by \citet{bochkov2025emergent} studies Transformer language
models with frozen representations derived from Unicode glyphs.
That approach retains a precomputed vocabulary-sized representation
table and incorporates visual information about token form.
Here we remove those visual features and use only binary identifiers,
with an implementation that need not store a vocabulary-sized input
table. The current results should be interpreted independently of
stronger semantic or mechanistic interpretations of frozen embeddings.

\paragraph{Token granularity.}
Character-aware, byte-level, and tokenization-free models change the
units from which text representations are constructed
\citep{kim2016character,xue2022byt5,clark2022canine,
tay2022charformer,yu2023megabyte}. Such models often retain learned
embeddings or learned front ends. We instead keep tokenization fixed
within the controlled comparison and change only the input
representation of token IDs. A binary code for a vocabulary token is
not the same as modeling the underlying text as bytes.

\paragraph{Weight tying and output-layer methods.}
Input/output weight tying shares vocabulary parameters across the two
interfaces \citep{press2017using}. Hierarchical and adaptive softmax
methods address the output-side vocabulary cost
\citep{morin2005hierarchical,mnih2009scalable,grave2017efficient}.
Our experiment leaves the output projection unchanged and untied.
Consequently, the reported parameter reduction is relative to an
\emph{untied} baseline. If an input table already shares its weights
with a retained output matrix, removing its input use does not in
general remove that matrix from storage.

\section{Method}
\label{sec:method}

\subsection{Minimal binary identity codes}

Let the vocabulary be $\mathcal{V}=\{0,\ldots,V-1\}$. Define
$$
K=\lceil\log_2 V\rceil,
\qquad
c(t)_j=
\left\lfloor\frac{t}{2^j}\right\rfloor\bmod 2,
\qquad j=0,\ldots,K-1.
$$
The map $c$ is injective. For $V=2^K$, it covers the entire binary
hypercube. Binary here denotes the alphabet $\{0,1\}$, not low-precision
quantization of learned embeddings or Transformer weights.

An optional affine recoding gives
$$
\widetilde c(t)=A c(t)\oplus b,
\qquad
A\in\GL(K,2),
\qquad b\in\{0,1\}^{K},
$$
where multiplication and addition are over $\Ftwo$. Invertibility of
$A$ preserves distinct code assignments. The canonical variant is the
special case $A=I_K$ and $b=0$.

We use one fixed affine assignment in the full-scale affine experiments.
Training seeds change model initialization and data order, not the
affine matrix or shift. This experiment therefore tests viability
under one non-canonical assignment rather than variability across
affine assignments.

\subsection{Fixed lift}

For $\dmodel=sK$, define
$$
R=
\begin{bmatrix}
I_K\\
I_K\\
\vdots\\
I_K
\end{bmatrix}
\in\mathbb{R}^{\dmodel\times K},
\qquad
x(t)=R\widetilde c(t).
$$
The code is cast to floating point and repeated $s$ times. The matrix
$R$ is a mathematical description of tiling; it need not be stored.

In the main experiments,
$$
V=65{,}536,\qquad K=16,\qquad \dmodel=1024,\qquad s=64.
$$
The lift has full column rank and introduces no trainable parameters.
All code-derived input vectors lie in a subspace of dimension at most
16. This rank statement concerns the input vectors, not all subsequent
hidden representations.

\subsection{What is learned}

A shared linear transformation $W$ acting directly on these inputs
would produce
$$
W x(t)=W R\widetilde c(t)=B\widetilde c(t),
\qquad B=WR.
$$
Thus the first learned transformations can be understood as learning
shared basis responses to fixed binary coordinates. This observation
connects the construction directly to compositional representations.
It does not make the parameterization equivalent to a free
$V\times\dmodel$ lookup, since token representations cannot be adjusted
independently through input rows.

The full Transformer also contains normalization, attention,
nonlinearities, and residual connections. The model is therefore
not restricted to a rank-16 representation throughout the network.
We make no claim that continuous representation learning has been
eliminated or that the input rank bound applies to the final model.

\subsection{Implementation and accounting}
\label{subsec:implementation}

The canonical checkpoint stores a frozen
$65{,}536\times16$ code tensor in an embedding module and tiles
the retrieved code at runtime. During training, the tensor is
non-trainable and excluded from the optimizer. It contains
1,048,576 stored scalar values, not zero storage.

The affine checkpoint extracts bits from token IDs and applies
the fixed transformation at runtime. The matrix, shift, and bit
positions are non-trainable buffers. There is no input embedding table.

A frozen lookup and on-the-fly computation are functionally equivalent
when they implement exactly the same map, including special-token
handling. The published canonical audit zeroes the padding row.
If the padding ID is zero, this leaves the canonical map unchanged.
If it is nonzero, the padding row and token ID zero share a code;
global injectivity then does not hold over all vocabulary entries.
The exact identity and collision checks needed to distinguish these
cases are discussed in \Cref{app:audit}. The mathematical results
in \Cref{app:proofs} apply to the unmodified injective coding rule.

\section{Experimental Setup}
\label{sec:setup}

\subsection{Architecture and parameter counts}

All paper models use 32 decoder-only Transformer layers,
model width 1024, 32 attention heads, context length 1024,
rotary positional encodings, GELU activations, and an untied
trainable output projection. The tokenizer and vocabulary are shared.

\begin{table}[t]
\centering
\caption{Parameter accounting for the paper models. Model-wide counts
are rounded to millions and follow the reported training-model counts.
The frozen-code column counts stored non-trainable embedding values.
It does not include other buffers, such as rotary caches or the small
affine-code buffers. Thus the last column is a parameter-tensor count,
not an exact checkpoint-size or peak-memory measurement.}
\label{tab:parameters}
\small
\setlength{\tabcolsep}{5pt}
\begin{tabular}{lrrrr}
\toprule
Input interface &
\shortstack{Trainable\\input (M)} &
\shortstack{All trainable\\parameters (M)} &
\shortstack{Frozen input\\code values (M)} &
\shortstack{All parameter\\values (M)}\\
\midrule
Learned input table & 67.109 & 537.2 & 0 & 537.2\\
Canonical Binary-16 & 0 & 470.1 & 1.049 & 471.2\\
Affine Binary-16 & 0 & 470.1 & 0 & 470.1\\
\bottomrule
\end{tabular}
\end{table}

The exact removed trainable input block contains
$$
65{,}536\times1024=67{,}108{,}864
$$
parameters, approximately 12.5\% of the baseline trainable count.
The unchanged output projection remains vocabulary-sized and trainable.
We do not reallocate the removed parameters into a larger backbone.

\subsection{Data and training}

Models are trained from scratch on a mixture of FineWeb-Edu
\citep{penedo2024fineweb} and Cosmopedia
\citep{benallal2024cosmopedia}, sampled with probabilities 0.8 and
0.2, respectively. The held-out split is constructed at the document
level before tokenization and shared across variants. Documents
assigned to validation are excluded from the training stream.

The deterministic 65,536-token tokenizer is constructed before
model training and kept fixed. It contains reserved special tokens,
Unicode-compatible assignments, and frequent multi-character tokens.
The controlled comparisons therefore do not change the tokenizer
along with the input interface.

Each full-scale variant uses three independent training seeds.
Seeds affect parameter initialization and data ordering.
Canonical and learned-input runs see 17.093B tokens per seed.
Affine runs see 16.276B tokens per seed because of scheduling
constraints. Architecture and training hyperparameters are otherwise
shared. Additional details are in \Cref{app:training}.

\subsection{Validation metrics}

We report arithmetic means of validation loss and perplexity across
the three training seeds. Perplexity is exponentiated per run before
averaging. For metric $m$, the relative seed range is
$$
r(m)=
\frac{\max_s m_s-\min_s m_s}
{\frac{1}{3}\sum_{s=1}^{3}m_s}.
$$
A seed range is a descriptive statistic, not a confidence interval.
Comparing a between-model gap with a seed range is not a significance
test, an equivalence test, or a non-inferiority test.

Validation perplexity uses the paper tokenizer and held-out training
distribution. It is distinct from the public-benchmark perplexities
reported below.

\subsection{Standardized base-checkpoint evaluation}
\label{subsec:base_protocol}

Released base checkpoints are evaluated with EleutherAI LM Evaluation
Harness \texttt{v0.4.10} \citep{evalharness}. The model-card protocol
specifies:
\begin{itemize}
    \item Hugging Face backend \texttt{hf};
    \item maximum context length 1024;
    \item \texttt{add\_bos\_token=False};
    \item no chat template and no SFT or instruction checkpoint;
    \item likelihood-based evaluation rather than sampled answer generation;
    \item harness seeds \texttt{0,1234,1234,1234}.
\end{itemize}

The suite includes HellaSwag \citep{zellers2019hellaswag},
ARC-Easy and ARC-Challenge \citep{clark2018arc},
PIQA \citep{bisk2020piqa}, WinoGrande
\citep{sakaguchi2020winogrande}, OpenBookQA
\citep{mihaylov2018openbookqa}, CommonsenseQA
\citep{talmor2019commonsenseqa}, MMLU
\citep{hendrycks2021mmlu}, LAMBADA
\citep{paperno2016lambada}, and WikiText
\citep{merity2017pointer}. MMLU is reported in both 0-shot and
5-shot settings.

The evaluation table reproduces the released checkpoint summaries.
It is not a three-training-seed benchmark average. The reported
uncertainty entries must not be confused with the training-seed
ranges used for validation. The summaries do not specify all
task-level metadata; \Cref{app:eval_reporting} records this
reproducibility limitation.

\paragraph{External references.}
SmolLM2-135M and SmolLM2-360M are base-model references
\citep{benallal2025smollm2,smollm2135card,smollm2360card}.
Their documented training budgets are approximately 2T and 4T
tokens, respectively, compared with approximately 16--17B tokens
for the paper models. They also differ in tokenizer, architecture,
training mixture, and optimization. Their results contextualize
absolute capability; they do not isolate an input-interface effect.

\section{Results}
\label{sec:results}

\subsection{Held-out language modeling}

\begin{table}[t]
\centering
\caption{Main validation comparison. Values are arithmetic means across
three training seeds. All variants share the tokenizer, backbone,
output head, and training recipe; the affine variant has a shorter
training budget. Lower loss and perplexity are better. The relative
seed range describes variation in validation perplexity, not uncertainty
in a formally estimated between-model effect.}
\label{tab:validation}
\small
\begin{tabular}{lrrrr}
\toprule
Input interface &
\shortstack{Tokens/seed\\(B)} &
\shortstack{Mean\\loss} &
\shortstack{Mean\\PPL} &
\shortstack{Relative\\seed range}\\
\midrule
Learned input table & 17.093 & 0.893 & 2.44 & 4.8\%\\
Canonical Binary-16 & 17.093 & 0.859 & 2.36 & 4.6\%\\
Affine Binary-16 & 16.276 & 0.871 & 2.39 & 4.5\%\\
\bottomrule
\end{tabular}
\end{table}

\Cref{tab:validation} shows that the fixed-code models learn the
held-out distribution without a trainable input table. The canonical
model has a lower mean validation perplexity than the learned-input
baseline, 2.36 versus 2.44, at equal token budgets. This is an
approximately 3.3\% reduction in the reported mean.

The observed seed ranges are of similar scale to this difference.
We report the difference descriptively and do not infer statistical
superiority or equivalence from the ranges. The supported observation
is that this training recipe yields comparable validation quality
after removing the trainable input block.

The affine model obtains mean perplexity 2.39 after 16.276B tokens
per seed. Its shorter training budget prevents a strict equal-budget
comparison. Nevertheless, it demonstrates successful learning with
one non-canonical code assignment and an implementation containing
no input lookup table.

\subsection{Standardized evaluation of base checkpoints}
\label{subsec:base_results}

\begin{table}[p]
\centering
\caption{Standardized evaluation of released base checkpoints with
LM Evaluation Harness \texttt{v0.4.10}. Accuracy entries are percentages;
perplexities and bits per byte are not percentages. Values and
$\pm$ entries are reproduced from the model-card evaluation summaries.
They are checkpoint-evaluation quantities, not means or variation
across the three training seeds.
The first three columns form the controlled architectural comparison.
SmolLM2 models are external references with different training budgets,
tokenizers, and architectures.
Higher accuracy is better; lower perplexity and bits per byte are better.
No across-task aggregate is used.}
\label{tab:base_eval}
\begingroup
\scriptsize
\setlength{\tabcolsep}{3pt}
\renewcommand{\arraystretch}{1.18}
\begin{tabular}{@{}lccccc@{}}
\toprule
& \multicolumn{3}{c}{Paper checkpoints}
& \multicolumn{2}{c}{External references}\\
\cmidrule(lr){2-4}\cmidrule(lr){5-6}
Metric &
\shortstack{Learned\\input table} &
\shortstack{Canonical\\Binary-16} &
\shortstack{Affine\\Binary-16} &
\shortstack{SmolLM2\\135M} &
\shortstack{SmolLM2\\360M}\\
\midrule
HellaSwag acc
& \score{28.49}{0.45}
& \score{29.04}{0.45}
& \score{29.04}{0.45}
& \score{35.36}{0.48}
& \score{43.05}{0.49}\\
HellaSwag acc\_norm
& \score{31.32}{0.46}
& \score{32.32}{0.47}
& \score{31.80}{0.46}
& \score{43.02}{0.49}
& \score{56.28}{0.50}\\
ARC-Easy acc
& \score{46.38}{1.02}
& \score{47.90}{1.03}
& \score{47.64}{1.02}
& \score{64.44}{0.98}
& \score{70.24}{0.94}\\
ARC-Easy acc\_norm
& \score{40.70}{1.01}
& \score{40.87}{1.01}
& \score{41.20}{1.01}
& \score{58.75}{1.01}
& \score{68.18}{0.96}\\
ARC-Challenge acc
& \score{20.39}{1.18}
& \score{19.62}{1.16}
& \score{21.33}{1.20}
& \score{28.07}{1.31}
& \score{36.26}{1.40}\\
ARC-Challenge acc\_norm
& \score{25.85}{1.28}
& \score{26.19}{1.28}
& \score{24.83}{1.26}
& \score{29.61}{1.33}
& \score{38.05}{1.42}\\
PIQA acc
& \score{62.35}{1.13}
& \score{62.57}{1.13}
& \score{62.68}{1.13}
& \score{68.44}{1.08}
& \score{71.38}{1.05}\\
PIQA acc\_norm
& \score{60.61}{1.14}
& \score{62.08}{1.13}
& \score{60.94}{1.14}
& \score{68.39}{1.08}
& \score{71.82}{1.05}\\
WinoGrande acc
& \score{50.20}{1.41}
& \score{50.12}{1.41}
& \score{50.43}{1.41}
& \score{52.57}{1.40}
& \score{59.35}{1.38}\\
OpenBookQA acc
& \score{18.40}{1.73}
& \score{17.20}{1.69}
& \score{17.60}{1.70}
& \score{22.00}{1.85}
& \score{24.80}{1.93}\\
OpenBookQA acc\_norm
& \score{29.20}{2.04}
& \score{31.00}{2.07}
& \score{29.40}{2.04}
& \score{32.60}{2.10}
& \score{37.80}{2.17}\\
CommonsenseQA acc
& \score{20.31}{1.15}
& \score{19.90}{1.14}
& \score{20.23}{1.15}
& \score{19.90}{1.14}
& \score{21.05}{1.17}\\
MMLU 0-shot
& \score{24.13}{0.36}
& \score{23.86}{0.36}
& \score{24.11}{0.36}
& \score{24.24}{0.36}
& \score{25.47}{0.37}\\
MMLU 5-shot
& \score{25.68}{0.37}
& \score{25.60}{0.37}
& \score{25.66}{0.37}
& \score{25.39}{0.37}
& \score{25.05}{0.37}\\
\midrule
LAMBADA accuracy
& \score{22.38}{0.58}
& \score{21.23}{0.57}
& \score{21.99}{0.58}
& \score{42.97}{0.69}
& \score{53.31}{0.70}\\
LAMBADA perplexity
& \score{95.14}{4.01}
& \score{101.74}{4.27}
& \score{100.61}{4.17}
& \score{19.06}{0.63}
& \score{9.38}{0.27}\\
WikiText word perplexity
& 81.04 & 74.87 & 76.17 & 25.53 & 18.84\\
WikiText byte perplexity
& 2.27 & 2.24 & 2.25 & 1.83 & 1.73\\
WikiText bits/byte
& 1.19 & 1.16 & 1.17 & 0.87 & 0.79\\
\bottomrule
\end{tabular}
\endgroup
\end{table}

\paragraph{Controlled comparison.}
The paper checkpoints show task-dependent differences rather than a
uniform ordering. Canonical Binary-16 has higher normalized HellaSwag
accuracy than the learned-input baseline, 32.32 versus 31.32, and
higher normalized PIQA accuracy, 62.08 versus 60.61. The learned-input
baseline has higher LAMBADA accuracy, 22.38 versus 21.23, and lower
LAMBADA perplexity, 95.14 versus 101.74. Canonical Binary-16 has
lower WikiText word perplexity, 74.87 versus 81.04.
The affine checkpoint generally remains close to the other paper
checkpoints.

These are descriptive comparisons of released checkpoints.
No paired significance analysis or multi-training-seed benchmark
analysis is reported, so small differences should not be treated
as resolved gains or losses.

\paragraph{Absolute capability and chance references.}
The benchmarks also delimit the result. PIQA accuracy is above the
50\% reference for uniform choice between two options, and HellaSwag
accuracy is above the 25\% reference for four options. In contrast,
WinoGrande is near its 50\% uniform-choice reference, CommonsenseQA
is near its 20\% five-option reference, and MMLU is near its 25\%
four-option reference. These analytical references are not empirical
majority-class baselines and are not significance tests.

Thus, the evaluations demonstrate nontrivial performance on some
tasks, but not strong general knowledge or reasoning across the suite.
Near-chance scores cannot establish meaningful capability parity on
the corresponding tasks. The results do not substantiate claims
about multilingual, coding, or mathematical performance.

\paragraph{External references.}
SmolLM2-135M and SmolLM2-360M perform substantially better on many
tasks, including HellaSwag, ARC-Easy, PIQA, and LAMBADA.
Their much larger training budgets and other design differences
prevent attribution of those gaps to input embeddings.
The comparison is nevertheless informative about the absolute quality
of the paper checkpoints and the limited regime in which the input
intervention has been tested.

\paragraph{Interpreting perplexity.}
The validation perplexities near 2.4 and the WikiText word perplexities
near 75--81 use different evaluation distributions and normalization
units. They should not be compared numerically as if they measured the
same quantity.

Across tokenizers, raw token-normalized perplexity is generally not
directly comparable. Word- and byte-normalized negative log-likelihood
use common external units and are more informative for cross-tokenizer
comparison, provided text preprocessing and scoring conventions agree.
They can still be affected by context segmentation and tokenizer-dependent
context coverage. The common limit of 1024 tokens does not imply the
same number of characters or bytes of context for every model.
The primary architectural comparison therefore remains among the
three paper checkpoints.

\FloatBarrier

\section{Discussion}
\label{sec:discussion}

\paragraph{A result about parameterization, not the disappearance of embeddings.}
The experiment shows that independently trainable input rows can be
removed in the studied setting while retaining comparable observed
validation quality and broadly similar checkpoint-level benchmark
behavior. The model still learns token-dependent continuous features
through shared Transformer parameters.

The contrast is between token-specific input storage and shared
computation over fixed identifiers. Both implement mappings from
discrete symbols to learned internal representations. The present
results do not locate all semantic information in a particular layer
or establish a general theory of how meaning is represented.

\paragraph{What the affine experiment establishes.}
One fixed invertible affine recoding produces a viable model.
This weakens an explanation that requires the exact canonical binary
assignment. It does not establish invariance to token-ID permutations,
all affine maps, or all injective codes. Affine transformations are a
structured subset of possible code assignments, and their variability
is not measured here.

\paragraph{Parameter savings and weight tying.}
Removing the untied learned input block reduces trainable parameters
by approximately 12.5\%. It also removes gradients and optimizer state
associated with that block. The canonical implementation still stores
a frozen code tensor; the affine implementation replaces that tensor
with small fixed buffers.

These counts are not measurements of end-to-end memory use, latency,
or throughput gains. Input embeddings are relatively inexpensive
computationally, and algorithmic code generation has its own cost.
Furthermore, if a baseline ties input and output weights, the retained
output matrix already supplies the storage otherwise attributed to
the input table. The same parameter savings cannot then be assumed.

\paragraph{Controls needed for mechanism.}
A learned $V\times16$ input table, a learned rank-16 factorization,
and fixed random continuous codes would help separate effects of
input rank, fixedness, binarity, and code geometry. Reallocating the
removed parameters to the backbone would address parameter-budget
efficiency. Multiple independently sampled affine maps would address
code-assignment sensitivity.

None of these comparisons is reported here. Consequently, we do not
claim that binarity explains the observed behavior, that this interface
is optimal, or that it is the best allocation of a fixed parameter
budget.

\section{Limitations}
\label{sec:limitations}

\paragraph{Scale and optimization.}
The main experiment covers one approximately 0.5B-parameter backbone,
one tokenizer, and one vocabulary size. Training uses a large effective
batch and only approximately 5,000 optimizer updates. The study does
not establish that any variant is fully optimized or near its
best achievable quality. Sharing a recipe controls the intervention
under that recipe but does not rule out interactions between the
recipe and the input parameterization.

\paragraph{Statistical scope.}
Three seeds provide limited information about training variability.
Seed ranges are not inferential uncertainty estimates. The public
benchmark table evaluates released checkpoints rather than a complete
multi-seed benchmark study. We do not claim formal non-inferiority,
equivalence, or superiority.

\paragraph{Task coverage and absolute quality.}
The benchmark suite is predominantly English and likelihood-based.
Several knowledge and reasoning results are near chance.
Multilinguality, code generation, mathematics, long-context use,
and instruction following are not established by these evaluations.
The external reference models expose substantial capability gaps.

\paragraph{Implementation and reproducibility.}
The canonical checkpoint stores a frozen table, and padding treatment
must be distinguished from the ideal injective code construction.
The available evaluation summaries do not enumerate every task
configuration or checkpoint revision. The optional SFT results also
lack a complete specification of the supervised mixture.
These reporting limitations are recorded in the appendices rather
than filled with inferred settings.

\paragraph{Remaining trainable vocabulary structure.}
The output projection remains a large learned vocabulary-dependent
matrix. The study therefore does not demonstrate a language model
without learned vocabulary parameters, nor does it test removal of
the output head.

\section{Conclusion}
\label{sec:conclusion}

We study a fixed minimal binary input interface for decoder-only
language models. With a vocabulary of 65,536 tokens, a 16-bit code
and fixed tiled lift remove 67.1M trainable input parameters from an
untied learned-input baseline.

Across three training seeds, fixed-code models obtain validation
perplexities comparable to the learned-input model under the reported
recipe. Standardized base-checkpoint evaluations add evidence beyond
in-distribution validation loss: differences among the paper models
are task-dependent, while absolute capability remains limited and
external reference models are stronger on many tasks.

The supported conclusion is narrow: in this regime, useful
next-token modeling does not require an independently trainable
input vector for every token. Continuous representations are still
learned by the Transformer, and the output projection remains
trainable. The results motivate further work on compact input
parameterizations without establishing universal superiority of
binary codes.

\section*{Availability}

The released paper checkpoints and their model cards are available at:
\begin{itemize}
    \item Learned input table:\\
    \url{https://huggingface.co/Bochkov/llm-fix-min-baseline-learned-input-table-model-classic}
    \item Canonical Binary-16:\\
    \url{https://huggingface.co/Bochkov/llm-fix-min-fixed-minimal-binary-code}
    \item Affine Binary-16, table-free implementation:\\
    \url{https://huggingface.co/Bochkov/llm-fix-min-affine-recoded-minimal-code-table-free}
\end{itemize}
The model cards include loading examples and input-interface audit
procedures. Loading these repositories uses custom model code;
users should inspect that code before enabling
\texttt{trust\_remote\_code=True}.

\section*{Broader Impact}

This work studies the parameterization of a language-model input
interface. Its main intended benefit is scientific: separating
token identity from independently trainable input storage provides
a controlled way to investigate architectural assumptions.
Removing input parameters may also reduce associated optimizer-state
requirements, although end-to-end efficiency is not measured here.

The models remain subject to familiar language-model risks, including
incorrect, biased, unsafe, and misleading outputs. Fixed input codes
do not address these risks. The released checkpoints are research
artifacts rather than deployment-ready systems. The work uses
existing datasets and does not introduce a new personal-data
collection procedure.

\bibliographystyle{plainnat}
\bibliography{main}

\clearpage
\appendix

\section{Training and Tokenizer Details}
\label{app:training}

All full-scale base runs use the architecture described in
\Cref{sec:setup}. Each variant is trained with three independent
initialization and data-order seeds. The affine code assignment
is fixed across the affine model's training seeds.

\begin{table}[ht]
\centering
\caption{Reported optimization settings.}
\label{tab:optimization}
\small
\begin{tabular}{lcc}
\toprule
Setting & Base pretraining & Optional SFT\\
\midrule
Precision & FP32 & FP32\\
Optimizer & AdamW & AdamW\\
Learning rate & $4.0\times10^{-4}$ & $2.0\times10^{-5}$\\
Schedule & Cosine with warmup & Cosine with warmup\\
Warmup steps & 150 & 100\\
Weight decay & 0.1 & 0.01\\
Gradient clipping & 1.0 & 1.0\\
Global micro-batch & 16 & 6\\
Gradient accumulation & 200 & 100\\
Context length & 1024 & 1024\\
Accelerators & $2\times$ H100 80GB & $1\times$ H100 80GB\\
\bottomrule
\end{tabular}
\end{table}

The base micro-batch is global across both devices, giving a per-device
micro-batch of 8. The nominal token capacity per base optimizer
update is
$$
16\times200\times1024=3{,}276{,}800.
$$
At this capacity, 17.093B tokens correspond to approximately 5,216
updates, and 16.276B tokens to approximately 4,967 updates.
Actual counts depend on the treatment of partial batches and token
accounting. The SFT token capacity is
$$
6\times100\times1024=614{,}400
$$
per update.

Reported observed throughput is approximately 19.8k tokens/s for base
pretraining and 5.5k tokens/s for SFT. These figures describe the training
configuration; they are not controlled throughput comparisons between
input interfaces.

The tokenizer is deterministic and constructed before model training.
It uses a vocabulary of 65,536 entries with special tokens,
Unicode-compatible assignments, and frequent multi-character tokens.
The same tokenizer is used for all controlled comparisons. Validation
documents are separated before tokenization; validation examples are
not constructed by slicing documents that also occur in the training
stream.

The large effective batch is shared by all variants. This makes the
comparison informative about the intervention under a common recipe,
but does not show that the batch size, learning rate, or number of
updates is optimal for any variant.

\section{Evaluation Reporting and Interpretation}
\label{app:eval_reporting}

\paragraph{Two different sources of variation.}
The validation table summarizes three independently trained models
per input variant. Its relative ranges describe training-seed
variation.

The standardized evaluation table reproduces the point estimates and
uncertainty entries reported for the released checkpoints.
The harness seed tuple controls evaluation randomness; it does not
represent four independently trained models.
The $\pm$ values should not be read as training-seed standard
deviations, seed ranges, or confidence intervals for an architectural
effect. Their precise estimator is determined by the corresponding
harness metric and raw result field.

\paragraph{Task metadata.}
The model-card summaries specify the harness version, backend,
context limit, BOS setting, evaluation seed tuple, absence of chat
formatting, and the two MMLU shot settings.
They do not enumerate exact task IDs and dataset revisions, all
other few-shot counts, task-specific prompt overrides, checkpoint
commit hashes, evaluation precision, batch sizes, or truncation
statistics.

Accordingly, this manuscript does not assign an unreported zero-shot
setting to every non-MMLU row. Reproducing those rows exactly requires
the original commands, resolved task configurations, and raw harness
results. In particular, LAMBADA and WikiText dataset variants should
be identified from those configurations rather than inferred from
display names.

\paragraph{Accuracy normalization.}
The labels \texttt{acc} and \texttt{acc\_norm} follow the harness.
They distinguish raw conditional-likelihood option scoring from
the task's length-normalized scoring rule. Normalization conventions
are task-dependent, so normalized and unnormalized accuracy are
reported separately rather than combined.

\paragraph{Common-unit language modeling.}
For total negative log-likelihood $L$ measured in nats,
word count $N_{\mathrm{word}}$, and byte count $N_{\mathrm{byte}}$,
the corresponding normalized quantities have the form
$$
\ppl_{\mathrm{word}}
=\exp\left(\frac{L}{N_{\mathrm{word}}}\right),
\qquad
\ppl_{\mathrm{byte}}
=\exp\left(\frac{L}{N_{\mathrm{byte}}}\right),
$$
and
$$
\operatorname{bits/byte}
=\frac{L}{N_{\mathrm{byte}}\log 2}
=\log_2\ppl_{\mathrm{byte}}.
$$
These identities assume consistent text, likelihood accounting,
and counting conventions. Values printed to two decimal places
need not satisfy them exactly after rounding.

\paragraph{Chance and majority references.}
For a task with $q$ options per item, independent uniform choice
has expected accuracy $1/q$. This differs from always selecting
the most frequent label, whose score depends on the evaluation
split. For varying numbers of options, the uniform-choice expectation
is the mean of the item-specific values $1/q_i$.
We do not substitute a majority baseline for a uniform-choice
reference or infer either from a model's observed accuracy.

\section{Input-Interface Audits}
\label{app:audit}

\paragraph{Learned input table.}
The baseline input weight has shape $65{,}536\times1024$ and
contains 67,108,864 trainable parameters in the training model.
Its entries are not restricted to binary values.

\paragraph{Canonical code storage.}
The canonical checkpoint stores a
$65{,}536\times16$ tensor with binary entries. The model-card audit
constructs the expected bit expansion for all IDs and then zeroes
the row indexed by \texttt{pad\_token\_id}. Equality with that tensor
establishes agreement with the documented storage rule.

It does not, by itself, establish that all stored rows are distinct.
If the padding ID is nonzero, zeroing its row creates a duplicate
of the canonical code for token ID zero. Consequently:
\begin{itemize}
    \item exact equality to the padding-modified tensor and global
    code uniqueness are separate audit properties;
    \item a table-free replacement must reproduce the same padding
    convention to be functionally identical;
    \item an injectivity claim must specify whether its domain is
    the full vocabulary or the active non-padding vocabulary.
\end{itemize}
A complete audit should report the padding ID, the number of unique
packed codes over the full vocabulary, and the corresponding count
over non-padding IDs. If padding is treated as an exception, its
attention and loss masking must also be documented.

Freezing during training additionally requires the tensor to be
excluded from optimization. Numerical checkpoint contents alone do
not prove the historical optimizer configuration, and a
\texttt{requires\_grad} flag after loading depends on the model
implementation.

\paragraph{Affine code generation.}
The affine model-card audit checks that
\texttt{get\_input\_embeddings()} returns \texttt{None}, that no
trainable parameters are registered under \texttt{input\_code},
and that the code-related state is stored as buffers.
The documented buffers are
\texttt{bit\_positions}, \texttt{A\_gf2}, and \texttt{b\_gf2}.

Enumerating every vocabulary ID and packing each generated bit
vector back into an integer provides a direct collision test.
The model card reports binary code shape $65{,}536\times16$ and
65,536 unique codes. These are the reported audit properties of the
release, distinct from the abstract proof below.

\section{Additional Matched-Depth Comparisons}
\label{app:depth}

\begin{table}[ht]
\centering
\caption{Supplementary matched-depth validation comparisons.
Relative changes are computed from the reported rounded
perplexities. The 32-layer row repeats the main comparison.
These results are not a controlled depth-scaling study.}
\label{tab:depth}
\small
\begin{tabular}{lrrr}
\toprule
Depth & Learned-input PPL & Canonical-code PPL & Relative change\\
\midrule
9 layers & 3.77 & 3.76 & $-0.3\%$\\
16 layers & 2.42 & 2.38 & $-1.7\%$\\
32 layers & 2.44 & 2.36 & $-3.3\%$\\
\bottomrule
\end{tabular}
\end{table}

The smaller matched-depth comparisons show similar observed
validation perplexity after replacing the input table.
The three-seed design described for the main 32-layer experiment
should not be assumed for the smaller runs, for which the supplied
summary does not provide seed-level statistics.
The table does not establish a depth-scaling law or isolate depth
from training-budget and optimization effects.

\section{Optional SFT Continuations}
\label{app:sft}

The original experiments also include supervised fine-tuning
continuations from the three base variants.
These results are retained as exploratory supplementary evidence
and are separate from the base-checkpoint harness evaluation.

\begin{table}[ht]
\centering
\caption{Previously reported optional SFT continuation results.
Accuracy columns are percentages. SQuAD gold-answer perplexity is
a conditional likelihood metric, not SQuAD exact match or F1.
These evaluations are not the standardized base-model protocol
of \Cref{tab:base_eval}.}
\label{tab:sft}
\small
\setlength{\tabcolsep}{5pt}
\begin{tabular}{lrrrrr}
\toprule
Input interface &
Val PPL & ARC-C & HellaSwag & BoolQ &
\shortstack{SQuAD\\gold PPL}\\
\midrule
Learned input table + SFT & 2.42 & 29.77 & 29.40 & 64.06 & 3.48\\
Canonical Binary-16 + SFT & 2.39 & 28.09 & 26.70 & 61.33 & 3.33\\
Affine Binary-16 + SFT & 2.44 & 31.44 & 26.80 & 64.06 & 3.65\\
\bottomrule
\end{tabular}
\end{table}

The evaluations use ARC, HellaSwag, BoolQ, and SQuAD
\citep{clark2018arc,zellers2019hellaswag,clark2019boolq,
rajpurkar2016squad}. The original protocol scores multiple-choice
options by conditional log-likelihood. The reported SFT
hyperparameters appear in \Cref{tab:optimization}.

The continuations use a shared supervised mixture and recipe across
variants, but the available report does not enumerate the mixture's
datasets and sampling weights, total SFT token counts, or complete
prompt and scoring configurations. It also does not provide
seed-level SFT uncertainties. These omissions limit independent
reproduction and comparison with other work.

The table shows that the fixed-code models can undergo continued
optimization and obtain finite downstream scores. It is not used
to establish strong adaptation quality, standardized base-model
capability, or statistical parity. In particular, BoolQ accuracy
must be interpreted against a split-specific majority baseline,
not only the 50\% uniform-choice reference for a binary task.

\section{Mathematical Properties of the Coding Rule}
\label{app:proofs}

The following results concern the deterministic unmodified code map.
Special-token overrides must be checked separately as described
in \Cref{app:audit}. These elementary properties establish identity
preservation at the coding interface; they do not imply any
particular learning performance.

\subsection{Minimal binary width}

\begin{proposition}[Minimal fixed-length binary identifier]
An injective map from $V$ symbols into $\{0,1\}^{K}$ requires
$$
K\ge\lceil\log_2 V\rceil.
$$
The canonical binary expansion attains this bound.
\end{proposition}

\begin{proof}
There are $2^K$ binary strings of length $K$. Injectivity requires
$2^K\ge V$, hence $K\ge\log_2 V$. Taking the smallest admissible
integer gives the bound. The $K$-bit expansions of
$0,\ldots,V-1$ are distinct whenever $V\le2^K$.
\end{proof}

This bound is specific to fixed-length binary strings.
An arbitrary real-valued scalar can label finitely many symbols
if precision is unconstrained, so the proposition is not a universal
lower bound on real-valued embedding dimension.

\subsection{Injectivity of affine recoding and lifting}

\begin{proposition}[Injectivity]
Let $c:\mathcal{V}\to\{0,1\}^{K}$ be injective,
$A\in\GL(K,2)$, $b\in\{0,1\}^{K}$, and
$R\in\mathbb{R}^{d\times K}$ have full column rank.
Then
$$
x(t)=R\bigl(Ac(t)\oplus b\bigr)
$$
is injective.
\end{proposition}

\begin{proof}
Suppose $x(t_1)=x(t_2)$. With
$z_i=Ac(t_i)\oplus b$, this gives
$$
R(z_1-z_2)=0.
$$
Full column rank implies $z_1=z_2$ as real and therefore binary
vectors. Applying the inverse affine map over $\Ftwo$ yields
$c(t_1)=c(t_2)$. Injectivity of $c$ gives $t_1=t_2$.
\end{proof}

The tiled identity lift has full column rank because
$$
R^\top R=sI_K.
$$
Thus tiling preserves distinct codes. This is an interface property;
it is not a claim that every subsequent network operation is
injective.

\begin{corollary}[Full-hypercube coverage]
If $V=2^K$, canonical codes followed by an invertible affine recoding
cover exactly $\{0,1\}^{K}$.
\end{corollary}

\begin{proof}
Canonical coding is a bijection from vocabulary IDs to the hypercube.
An invertible affine transformation is a bijection on that hypercube.
Their composition is therefore a bijection.
\end{proof}

For $K\ge2$, full coverage implies that each bit is one for half
the vocabulary and each pair of distinct bit coordinates takes
each binary pair equally often. This is balance under uniform
enumeration of vocabulary entries, not balance under the empirical
token-frequency distribution.

\subsection{Input rank}

\begin{proposition}[Rank of the lifted code set]
Let $\mathcal{X}=\{R\widetilde c(t):t\in\mathcal{V}\}$.
Then
$$
\dim\spanof(\mathcal{X})\le\rank(R)\le K.
$$
If the code set is the full hypercube and $\rank(R)=K$, equality
holds.
\end{proposition}

\begin{proof}
Every element of $\mathcal{X}$ belongs to $\col(R)$, establishing
the upper bound. The full hypercube contains all standard basis
vectors of $\mathbb{R}^{K}$, so its real span is
$\mathbb{R}^{K}$. Its image under $R$ therefore spans $\col(R)$,
which has dimension $K$.
\end{proof}

\subsection{Storage and implementation equivalence}

For a fixed map $\widetilde c$, one can store the compact tensor
$$
C[t]=\widetilde c(t)\in\{0,1\}^{K}
$$
and tile it after lookup, or compute the same bits at runtime.
The compact lookup stores $VK$ values rather than $Vd$.
A third, unnecessary implementation could store all lifted vectors
as a frozen $V\times d$ table.

Lookup and algorithmic implementations are functionally identical
only when they agree for every relevant ID, including padding and
other special tokens. This equivalence concerns numerical inputs,
not necessarily runtime or memory efficiency.

The affine implementation stores $K^2$ matrix entries, $K$ shift
entries, and bit-position metadata. With tiling implemented
algorithmically, its code specification scales as $O(K^2)$.
These are counts of fixed entries, not a claim that the implementation
packs them into single physical bits.

\section{Artifacts}
\label{app:manifest}

The released base-model checkpoints are available under
\url{https://huggingface.co/Bochkov/}:

\begin{itemize}
    \item \textbf{Learned input table:}\\
    \url{https://huggingface.co/Bochkov/llm-fix-min-baseline-learned-input-table-model-classic}

    \item \textbf{Canonical Binary-16:}\\
    \url{https://huggingface.co/Bochkov/llm-fix-min-fixed-minimal-binary-code}

    \item \textbf{Affine Binary-16, table-free implementation:}\\
    \url{https://huggingface.co/Bochkov/llm-fix-min-affine-recoded-minimal-code-table-free}
\end{itemize}

The model cards document loading instructions, input-interface
audit procedures, and the reported base-checkpoint evaluation results.

\end{document}